\documentclass[11pt]{article}

\usepackage[letterpaper,margin=1in]{geometry}
\usepackage[T1]{fontenc}
\usepackage[utf8]{inputenc}
\usepackage{newtxtext}

\usepackage{amsmath,amssymb,amsthm,mathtools}
\usepackage{newtxmath}
\usepackage{microtype}

\usepackage{booktabs}
\usepackage{tabularx}
\usepackage{array}
\usepackage{caption}
\usepackage{subcaption}
\usepackage{float}

\usepackage{enumitem}
\setlist{noitemsep,topsep=2pt,leftmargin=1.4em}

\usepackage[ruled,vlined,linesnumbered]{algorithm2e}

\usepackage{listings}
\usepackage{xcolor}
\definecolor{kw}{HTML}{0B5394}
\definecolor{cmt}{HTML}{6A737D}
\definecolor{str}{HTML}{0A7D33}
\definecolor{bg}{HTML}{F6F8FA}
\definecolor{accent}{HTML}{1F4E79}
\definecolor{soft}{HTML}{E8EEF6}
\definecolor{warm}{HTML}{F4ECDD}
\lstdefinestyle{cs}{
  basicstyle=\ttfamily\footnotesize,
  backgroundcolor=\color{bg},
  keywordstyle=\color{kw}\bfseries,
  commentstyle=\color{cmt}\itshape,
  stringstyle=\color{str},
  numberstyle=\tiny\color{cmt},
  breaklines=true,
  showstringspaces=false,
  columns=fullflexible,
  frame=single,
  framerule=0pt,
  framesep=6pt,
  aboveskip=8pt,belowskip=6pt,
  literate={`}{\textasciigrave}1
}
\usepackage{tikz}
\usetikzlibrary{positioning,arrows.meta,calc,fit,backgrounds,shapes.geometric,shapes.misc}

\usepackage[numbers,sort&compress]{natbib}
\usepackage[colorlinks=true,linkcolor=accent,citecolor=accent,urlcolor=accent]{hyperref}
\usepackage[nameinlink,capitalise]{cleveref}
\usepackage{authblk}

\theoremstyle{definition}
\newtheorem{definition}{Definition}
\theoremstyle{plain}
\newtheorem{proposition}{Proposition}

\newcommand{\code}[1]{\texttt{\small #1}}
\newcommand{\msc}[1]{\textsc{#1}}
\newcommand{\evsql}{\textsc{evidence sql}}
\DeclareMathOperator*{\argmin}{arg\,min}
\DeclareMathOperator*{\argmax}{arg\,max}
\newcommand{\given}{\,|\,}
\providecommand{\rupee}{\textsf{Rs}\,}

\tikzset{
  stage/.style={rounded corners=2pt, draw=accent, fill=soft, thick,
    align=center, font=\footnotesize\sffamily, minimum height=8mm, inner sep=3pt},
  proc/.style={rounded corners=2pt, draw=accent!80, fill=white, thick,
    align=center, font=\footnotesize\sffamily, minimum height=8mm, inner sep=3pt},
  store/.style={cylinder, shape border rotate=90, aspect=0.18, draw=cmt,
    fill=gray!10, align=center, font=\scriptsize\sffamily, minimum height=10mm,
    minimum width=20mm, inner sep=2pt},
  dec/.style={diamond, aspect=2, draw=accent!80, fill=warm, align=center,
    font=\scriptsize\sffamily, inner sep=1pt},
  skillb/.style={rounded corners=2pt, draw=accent, fill=warm, very thick,
    align=center, font=\footnotesize\sffamily, inner sep=4pt},
  flow/.style={-{Stealth[length=2mm]}, thick, draw=accent!85},
  dflow/.style={-{Stealth[length=2mm]}, thick, dashed, draw=cmt},
  lbl/.style={font=\scriptsize\sffamily\itshape, fill=white, inner sep=1pt}
}

\title{\bfseries From Question-First to Analyst-First:\\[2pt]
\large Domain-Expert Skills and Verified Knowledge Compilation\\
for Proactive Enterprise Analytics}

\author[ ]{Harmohit Singh}
\author[ ]{Rahul Sharma}
\affil[ ]{\normalsize CoreOps AI \quad\texttt{\small \{harmohit.singh,\;rahul.sharma\}@coreops.ai}}
\date{June 15, 2026}

\begin{document}
\maketitle

\begin{abstract}
\noindent
Conversational analytics systems overwhelmingly assume the user already has a
well-formed question, leaving a non-expert facing a blank query box on an
unfamiliar enterprise schema. Commercial ``proactive'' tools narrow this gap
only by detecting statistical anomalies over analyst-curated metric layers, and
academic next-question recommenders depend on query logs that a fresh dataset
lacks. We describe a production enterprise analytics system that inverts the
interaction model from \emph{question-first} to \emph{analyst-first} through two
coupled architectural ideas. First, a pluggable domain-expert \emph{skill}
abstraction: a folder-based, database-free subject-matter pack (a manifest,
per-stage prompt facets, keyword-routed reference libraries, standing-report
templates, and optional deterministic compute) that is auto-selected per
(client, dataset) by a deterministic schema-matching function and spliced as a
\emph{cross-cutting concern} into every stage of a streaming agentic pipeline,
the schema explorer, and the report engines, degrading to a strict no-op when
absent. Because a skill is a self-contained folder resolved deterministically and
added with no per-client engineering, the catalogue is open-ended: an extensible
\emph{marketplace} of domain experts rather than a fixed feature set. Second, an
offline \emph{knowledge-compilation} loop: a Claude-Code-style
agent probes the dataset's parquet via DuckDB (zero load on production), runs
critic-gated per-table convergence with self-healing retries, and data-validates
joins by value overlap, producing durable schema knowledge that, together with
the resolved skill, drives standing expert reports whose every published metric
is re-verified by independently re-executing its evidence SQL, and suggested
questions that mirror the report agenda. These close a \emph{proactive loop}:
reports surface numbers, the numbers seed questions, and a click launches a verified deep
dive, all before the query box is ever used. We give a formal model of the
injection operator and the compilation function, describe the fully-automatic
onboarding chain, and report illustrative evidence from a single live tenant:
deep exploration over $10/10$ tables with $22$ data-validated joins, seven
chartered-accountant standing reports, and roughly thirty domain questions
generated with no prior user query. We make no user-study or benchmark claims;
the contribution is the architecture and its defensibility.
\end{abstract}

\section{Introduction}
\label{sec:intro}

The dominant interface to enterprise data is now a text box. A decade of work on
natural-language interfaces to databases and, more recently, large-language-model
(LLM) text-to-SQL has driven translation accuracy upward on academic benchmarks
\citep{yu2018spider,li2023bird,hong2024survey} and into commercial products
\citep{aws2021quicksightq,google2025lookerca}. Yet every one of these systems
shares a hidden premise: \emph{the user already knows what to ask}. A perfect
translator from English to SQL does nothing for an analyst who is staring at a
freshly connected schema of eighty opaque tables and does not know which question
is worth a query. This is the \emph{cold-start} or \emph{blank-query-box} problem,
and it is a documented barrier to self-service business-intelligence adoption
\citep{lennerholt2021ssbi}. The hard part of analysis is not phrasing a question;
it is knowing which question a domain expert would ask of \emph{this} data.

Two families of systems try to supply the question rather than merely answer it,
and each cold-starts in its own way. \emph{Log-based next-question recommenders}
mine the queries of prior users to suggest the next step
\citep{chatzopoulou2009querie,wang2022nextstep}, but a new tenant has no logs,
so the recommender itself has nothing to recommend from on day one.
\emph{Automated insight discovery} sweeps a dataset for statistically salient
patterns \citep{ding2019quickinsights,demiralp2017foresight,vartak2015seedb}, but
``statistically interesting'' is not the same as ``what a chief financial officer
or a chartered accountant would act on,'' and these tools are domain-agnostic by
construction. The most proactive commercial offerings divide into two camps that
share the same blind spot. Tableau Pulse is bound to an \emph{analyst-curated
metric layer}: a human first defines the metrics and the system narrates their
movements \citep{tableau2025aboutpulse,tableau2025pulseinsighttypes}. ThoughtSpot
SpotIQ and Power BI Quick Insights are domain-agnostic statistical insight
miners: anomaly, correlation, and trend detection over a dataset
\citep{thoughtspot2025spotiq,microsoft2025quickinsights}. Neither camp carries a
\emph{portable domain expert} that, given a matching skill for the domain, can
approach an un-curated schema and author the agenda itself.

\paragraph{Thesis.} This paper describes a system that flips the interaction model
from question-first to analyst-first by treating domain expertise as a
\emph{first-class, pluggable, auto-selected software artifact} and by
\emph{compiling} durable, verified knowledge about a dataset \emph{before} the
user asks anything. The two ideas are coupled. A domain-expert \emph{skill} is a
self-contained folder (prompt fragments, reference libraries, standing-report
templates, deterministic compute) that the system splices into every node of its
agentic pipeline and offline engines. An offline \emph{knowledge-compilation} loop,
modeled on the way a coding agent explores a codebase, probes the data until an
independent critic is satisfied, validates joins against the data, and persists a
reusable knowledge base. The skill and the knowledge together drive (i) standing
expert reports whose every published number is re-verified by re-executing its
evidence SQL, and (ii) suggested questions drawn from the skill's review-cadence
playbook and grounded in the validated knowledge. The result is a closed
\emph{proactive loop}
(\cref{fig:proactive-loop}): the system surfaces \emph{numbers} the user did not
ask for, the numbers \emph{seed questions}, and a click launches a verified
\emph{deep dive}, a sequence that runs to completion before the blank box is ever
touched.

\paragraph{Why the fusion, not the parts, is the contribution.} Each ingredient in
isolation is incremental. Skill packaging echoes the modular ``agent skills'' and
subagent patterns now emerging as general developer tooling
\citep{anthropic2025agentskills,anthropic2025subagents}; agentic exploration and
critic-gated refinement are established
\citep{yao2023react,shinn2023reflexion,gou2023critic}; execution-grounded
verification has precedent \citep{gou2023critic,chen2023selfdebug}. The
contribution is the \emph{fusion}, deployed: a per-tenant, schema-auto-selected
domain-expert layer threaded as a cross-cutting concern through an entire
analytics stack, fused with a data-validated knowledge base and a report engine in
which a number that fails to reproduce against its own evidence SQL is suppressed
rather than published, instantiated as a closed proactive loop with a
fully-automatic onboarding chain. We return to this argument in
\cref{sec:discussion}.

\paragraph{Scope and honesty.} This is a single-deployment \emph{system} paper. We
report no controlled user study, no A/B test, and no public-benchmark leaderboard;
our deployment figures (\cref{sec:casestudy}) are illustrative observations from
one live tenant, reported for transparency rather than as measured wins. Where the
architecture makes a guarantee, we state its boundary precisely: the verification
layer certifies that a published number \emph{matches its declared evidence SQL}
within tolerance, not that the metric's business definition is semantically
correct.

\paragraph{Contributions.}
\begin{enumerate}
\item A pluggable, database-free \textbf{domain-expert skill abstraction}
(\cref{sec:skill}) that injects opinionated expertise as a cross-cutting concern
at every node of an agentic analytics pipeline, the schema explorer, and both
report engines, while degrading to an exact no-op when no skill is active.
\item A deterministic, source-agnostic \textbf{skill-selection algorithm} that
resolves one active skill per (client, dataset) from schema structure, including
distinctive \emph{cell values} that let expert framing fire on opaque schemas
(SAP T-tables, Tally exports) where names alone fail, with no query logs and no
per-client engineering.
\item An offline \textbf{knowledge-compilation loop} (\cref{sec:explore}): a
JSON-action probe agent over read-only DuckDB-on-parquet, gated by a fast
deterministic coverage critic, wrapped in a self-healing convergence loop whose
work list provably shrinks, that data-validates joins by value overlap and
persists durable, resumable schema knowledge.
\item An \textbf{execution-grounded report-synthesis engine}
(\cref{sec:reports}) in which every published metric carries an evidence query
that an independent, non-LLM critic re-executes and tolerance-matches, so a
number that does not reproduce against its evidence query is suppressed rather than
shown, and a durable recipe drives zero-LLM, zero-kernel refresh and drill-down.
\item The \textbf{proactive-analyst loop and its fully-automatic onboarding chain}
(\cref{sec:proactive}), in which the \emph{same} compiled knowledge and skill
power reports, questions, and live chat.
\item A \textbf{formal model} (\cref{sec:formal}) of the injection operator (with
a no-op/backward-compatibility property) and the compilation function (with a
termination argument), plus an honest single-deployment account
(\cref{sec:casestudy,sec:limits}).
\end{enumerate}

\section{Related Work}
\label{sec:related}

\paragraph{Text-to-SQL and the translation framing.}
Semantic parsing of natural language to SQL has matured from cross-domain
benchmarks \citep{yu2018spider} to knowledge-grounded, value-sensitive benchmarks
\citep{li2023bird} and broad LLM-era surveys
\citep{hong2024survey,shi2024t2sqlsurvey}. Crucially, \emph{Spider 2.0} reframes
the task around \emph{real enterprise} workflows (dozens of tables, dialects, and
external knowledge) and shows that translation alone is far from sufficient
\citep{lei2024spider2}, motivating both schema scouting and an external knowledge
layer of the kind we compile. All of this work, however, takes the question as
given.

\paragraph{Disambiguating a stated question vs.\ the prior question of what to ask.}
A complementary line helps users refine a question they have already posed: NaLIR
interactively resolves parse ambiguity \citep{li2014nalir}, while DataTone and
Eviza manage ambiguity in natural-language visualization interfaces
\citep{gao2015datatone,setlur2016eviza}. These address \emph{how to phrase}; we
address \emph{what to ask in the first place}.

\paragraph{Cold-start and next-question recommendation.}
QueRIE recommends queries from collaborative analysis of prior sessions
\citep{chatzopoulou2009querie}, and \citet{wang2022nextstep} recommend the next
natural-language step in an interactive session. Both are log- or
session-dependent and therefore cold-start on a fresh dataset, precisely the gap
a portable domain-expert prior closes. Recent work generates questions directly
from tabular data for conversational exploration \citep{chaudhuri2024autoqg}, the
closest prior art to our question generator. It derives questions from table
statistics and content; ours differ in two ways that the rest of this paper makes
concrete. First, the agenda is seeded by the resolved skill's KPI catalogue and
review-cadence playbook (\cref{sec:proactive}), so the questions reflect what a
role would ask rather than what is statistically describable. Second, they are
generated against the compiled, data-validated schema (\cref{sec:explore}), so a
suggested question is answerable by the same joins the live pipeline will use, an
answerability constraint statistics-only generation does not impose.

\paragraph{Automated insight discovery and mixed-initiative analytics.}
Mixed-initiative interaction is a long-standing ideal \citep{horvitz1999mixedinit}.
Visualization recommenders \citep{vartak2015seedb,wongsuphasawat2016voyager,
wongsuphasawat2017voyager2,siddiqui2017zenvisage} and insight miners
\citep{ding2019quickinsights,demiralp2017foresight,ma2021metainsight,cui2019datasite}
proactively surface salient patterns, and recent LLM agents push toward proactive
assistance \citep{zhao2025proactiveva,lu2024proactiveagent} amid a renewed survey
interest in the automation--initiative balance \citep{monadjemi2025scopingreview}.
These systems optimize statistical interestingness; our agenda is set by a domain
expert and is verified against the data.

\paragraph{LLM data-science and multi-agent systems.}
Reasoning-and-acting and reflective agents \citep{yao2023react,shinn2023reflexion},
multi-agent frameworks \citep{wu2023autogen,hong2024metagpt}, and data-science
agents \citep{hong2024datainterpreter} together with their evaluation benchmarks
\citep{hu2024infiagent} establish the agentic substrate we build on.
Visualization- and dashboard-generation agents
\citep{dibia2023lida,ma2023insightpilot,zhang2025datatodashboard} target the
artifact; we target the \emph{closed loop} that decides which artifacts to make and
re-verifies their numbers against evidence SQL.

\paragraph{Domain-knowledge injection.}
Retrieval-augmented generation and tool use inject external knowledge into LLM
inference \citep{lewis2020rag,gao2024ragsurvey,schick2023toolformer}; in the SQL
setting, Knowledge-to-SQL trains a single ``data-expert'' model to supply
knowledge \emph{upstream} to improve SQL generation \citep{hong2024knowledge2sql}, in
contrast to our \emph{downstream} evidence-SQL publish gate. Semantic/metrics layers
\citep{dbt2024semanticlayer,looker2025lookml} encode metric definitions but require
human curation per organization. Modular ``agent skills'' and subagents package
capability and orchestration as general developer tooling
\citep{anthropic2025agentskills,anthropic2025subagents}.
We specialize skill packaging into a per-tenant, schema-auto-selected
\emph{analytics} skill that threads the same expertise across an entire pipeline
and both report engines, and we derive the agenda from raw schema rather than a
pre-curated metric layer.

\paragraph{Verifier-/critic-gated loops and agent memory.}
Verifier ranking \citep{cobbe2021verifiers}, process supervision
\citep{lightman2023verify}, self-consistency \citep{wang2022selfconsistency}, and
self-refinement \citep{madaan2023selfrefine} improve reliability, while
\citet{huang2023cannotselfcorrect} caution that \emph{intrinsic} self-correction is
unreliable, motivating \emph{execution-grounded} critics
\citep{gou2023critic,chen2023selfdebug} and LLM-as-judge evaluation
\citep{zheng2023llmjudge}; see \citet{pan2024correctionsurvey} for a survey of
automated-correction strategies. Persistent agent memory and
skill libraries \citep{park2023generativeagents,packer2023memgpt,wang2023voyager}
inform our durable knowledge base. We also draw on the relational data-profiling
literature \citep{abedjan2015profiling} for what an automated explorer should
measure. Our distinctive move is to fuse an \emph{opinionated} expert (known to
risk domain-correlated persona bias \citep{tseng2024persona}) with a
\emph{deterministic} execution check, so the opinion frames the analysis but
cannot fabricate the number.

\paragraph{Commercial proactive analytics.}
Question-first natural-language products
\citep{aws2021quicksightq,aws2025quicksightqembedded,google2025lookerca},
metric-bound proactive products
\citep{tableau2025aboutpulse,tableau2025pulseinsighttypes,dbta2024tableaueinstein},
and domain-agnostic statistical insight miners
\citep{thoughtspot2025spotiq,microsoft2025quickinsights} define the landscape. The
white space they leave (proactivity over an un-curated schema, for a domain a
skill covers, driven by a portable domain expert and backed by execution
verification) is the space this system occupies (\cref{tab:positioning}).

\section{Problem Framing and Formalization}
\label{sec:formal}

We model an enterprise analytics platform that is \emph{multi-tenant} (everything
is scoped by a client identifier with no defaults) and serves two regimes: a
\emph{reactive} live pipeline that answers a query, and a \emph{proactive} offline
phase that produces artifacts before any query. The formalization makes precise
(a) what a skill is and how it is injected, (b) what knowledge compilation
computes and why it terminates, and (c) the admissibility contract that lets a
report publish a number.

\subsection{Sets and objects}
Let $C$ be the set of clients and $D$ the set of datasets. For each pair $(c,d)$
the dataset has a schema $\Sigma(c,d)=(T,\mathit{Col},\mathit{Val})$, where $T$ is
the bag of table names, $\mathit{Col}$ the bag of column names, and
$\mathit{Val}$ a sampled bag of cell values from the top rows. The live pipeline
is an ordered tuple of stages
$P=(s_1,\dots,s_6)=(\textsf{guard},\textsf{resolve},\textsf{scout},
\textsf{collate},\textsf{execute},\textsf{narrate})$, each a function on a run
context $x$, i.e.\ $s_i:X\to X$. $P$ is the \emph{maximal} stage sequence: the route
resolved in \textsf{resolve} ($\textsf{cache\_hit}$, $\textsf{simple}$, or
$\textsf{complex}$) may bypass \textsf{scout} and \textsf{collate}, but because
injection is per-stage the results below hold on whichever sub-sequence executes.

\begin{definition}[Skill]\label{def:skill}
A \emph{skill} is a tuple
$k=(\mathit{slug},\mu,F,R,\kappa,\rho,Y,L,\mathit{cfg})$ where
$\mu$ is the auto-detect spec
$(\textsf{any\_table},\textsf{any\_column},\textsf{any\_value},
\textsf{keywords},\textsf{min\_score})$;
$F:N\to\Sigma^*$ maps a node name in
$N=\{\textsf{router},\textsf{scout},\textsf{coder},\textsf{chart},
\textsf{narrator},\textsf{report\_planner},\textsf{report\_analyst}\}$ to its
\emph{facet} text ($\varepsilon$ if absent);
$R:\textit{stem}\to\Sigma^*$ are reference documents with a keyword map
$\kappa:\textit{stem}\to 2^{W}$ over a vocabulary $W$;
$\rho$ holds routing fields (KPI/playbook/schema-mapping reference stems, a
persona example, core references, and a \textsf{force\_ds} flag);
$Y$ are standing-report templates; $L$ are deterministic library modules; and
$\mathit{cfg}$ holds text budgets. The \emph{registry}
$K=\{k:\text{\code{SKILL.md} present}\}$ is a pure function of disk.
\end{definition}

\subsection{Selection: deterministic and fail-open}
Write $a\sqsubseteq B$ for ``$a$ occurs as a lowercased substring in blob $B$.''
Define, for a skill $k$ and schema $\Sigma$, the hit counts
\[
t=|\{a\in\textsf{any\_table}:a\sqsubseteq T\}|,\quad
c=|\{a\in\textsf{any\_column}:a\sqsubseteq \mathit{Col}\}|,
\]
\[
v=|\{a\in\textsf{any\_value}:a\sqsubseteq \mathit{Val}\}|,\quad
w=|\{a\in\textsf{keywords}:a\sqsubseteq (T\cup \mathit{Col})\}|,
\]
and the score
\begin{equation}
\mathrm{score}(k,\Sigma)=2t+c+2v+w .
\label{eq:score}
\end{equation}
A skill \emph{qualifies} iff
\begin{equation}
\big(t\ge 1 \;\lor\; v\ge 1\big)\;\land\;\mathrm{score}(k,\Sigma)\ge
\textsf{min\_score}(k).
\label{eq:qualify}
\end{equation}
The table-or-value requirement in \cref{eq:qualify} is what prevents a lone
coincidental column (a stray \code{vendor} column on an unrelated finance schema)
from auto-triggering an expert; the \emph{value} disjunct is what lets a
distinctive cell string (a Tally accounting group such as ``Sundry Debtors'')
stand in for a table-name hit on schemas whose \emph{names} are opaque. In the live
path $\sigma$ is evaluated with $\mathit{Val}=\varnothing$ (the chat selector scores
over names only); the \emph{value} disjunct is exercised by the offline detector over
sampled cells (\cref{sec:explore}) and folded back into the live assignment by the
onboarding chain's persistence edge (\cref{sec:proactive}). The active
skill is
\begin{equation}
\sigma(c,d)=
\begin{cases}
\mathrm{assign}(c,d) & \text{explicit assignment exists (}\bot\text{ if disabled)}\\[2pt]
\displaystyle\argmin_{k\ \text{qualifying}}\big(-\mathrm{score}(k,\Sigma),\,
-\textsf{min\_score}(k),\,\mathit{slug}(k)\big) & \text{else, if any qualifies}\\[6pt]
\bot & \text{otherwise,}
\end{cases}
\label{eq:select}
\end{equation}
where the $\argmin$ is over the lexicographic order shown: highest score, then the
skill that cleared a stricter bar (higher \textsf{min\_score}, the stronger claim
relative to its own threshold), then the alphabetically-first slug, making selection
deterministic and independent of registry order. $\sigma$ is gated by a global
flag, memoized with a short positive and a deliberately shorter negative cache
window, and yields $\bot$ on \emph{any} error; it is fail-open by
construction. The resolved slug is written once into $x.\textsf{active\_skill}$ in
stage $s_1$ and carried as a lightweight string thereafter.

\subsection{The injection operator and its no-op property}
For a stage $s_i$ with active skill $k=\sigma(c,d)$, define the injected prompt
\begin{equation}
\mathrm{inj}_i(\mathit{prompt},x)=\mathit{prompt}\;\oplus\;\mathrm{render}_i(k,x),
\label{eq:inject}
\end{equation}
where $\oplus$ splices facet/reference text at the stage's anchor (for the coder,
\emph{before} any custom-prompt override marker) and $\mathrm{render}_i(\bot,x)
=\varepsilon$. As an example, the coder render composes the methodology and chart
facets, a keyword-selected reference, and a shared reasoning-discipline block:
\[
\mathrm{render}_{\textsf{coder}}=
H\oplus F(\textsf{coder})\oplus F(\textsf{chart})\oplus
\mathrm{refsel}(k,q)\oplus \mathrm{DISC},
\]
with $\mathrm{refsel}(k,q)=\argmax_{\textit{stem}}|\{w\in\kappa(\textit{stem}):
w\sqsubseteq q\}|$ over the (enhanced) question $q$, truncated to a chat budget;
the offline report path uses the \emph{same} routing $\kappa$ but admits multiple
\emph{full} references under a larger budget.

\begin{proposition}[No-op additivity / backward compatibility]\label{prop:noop}
Assume each stage's dependence on the active skill is confined to $\mathrm{inj}_i$
(the slug enters $s_i$ only through the injected prompt). If $\sigma(c,d)=\bot$ then
for every stage $i$ and every $\mathit{prompt}$,
$\mathrm{inj}_i(\mathit{prompt},x)=\mathit{prompt}$, and the skill-augmented
pipeline equals the baseline pipeline on every executed sub-sequence of $P$.
\end{proposition}
\begin{proof}
$\mathrm{render}_i(\bot,x)=\varepsilon$ for all $i$ by definition and
$p\oplus\varepsilon=p$, so each stage receives an unchanged prompt; by the factoring
hypothesis the slug affects $s_i$ only through that prompt, so each $s_i$ is
unchanged as a function. Composing unchanged stages along any executed route leaves
the pipeline unchanged.
\end{proof}
\Cref{prop:noop} is the property that makes the entire subsystem safe to ship:
with no skill active, or any failure in resolution, the product is byte-identical
to its skill-free baseline.

\subsection{Offline knowledge compilation}
Let $B$ be a read-only DuckDB-on-parquet backend over $\Sigma$. For each table
$\tau$, a probe agent $A$ runs a self-terminating loop that emits actions from a
tool set $\textsf{Tools}$ comprising \textsf{get\_schema}, \textsf{profile\_columns},
\textsf{sample\_values}, \textsf{run\_sql}, \textsf{record\_finding}, and
\textsf{submit};
a deterministic coverage critic $\gamma_{\mathrm{cov}}$ maps a submission to
$\{\textsf{ok},\textsf{gaps}\}$, and an \emph{optional} LLM depth critic
$\gamma_{\mathrm{dep}}$ (fail-open, returning \textsf{deep} on any error) may
demand a deeper pass; $\mathrm{converged}(\tau)$ holds when $\gamma_{\mathrm{cov}}$
is satisfied and, where enabled, $\gamma_{\mathrm{dep}}$ is too. The compiler iterates
over the set of unconverged tables:
\begin{equation}
W_0=T,\qquad
W_{r+1}=\{\tau\in W_r : \lnot\,\mathrm{converged}(\mathrm{probe}(\tau,
\mathit{feedback}_r(\tau)))\},\quad r=1,\dots,R,
\label{eq:fixpoint}
\end{equation}
with a round cap $R$ (and a wall-clock backstop), where a \emph{keep-best} policy
guarantees a table never regresses once converged.

\begin{proposition}[Termination]\label{prop:term}
The sequence $(W_r)$ is monotone decreasing ($W_{r+1}\subseteq W_r$) and the loop
terminates in at most $R$ rounds. The compiled depth is \msc{deep} iff $W_R=
\emptyset$, else \msc{partial}.
\end{proposition}
\begin{proof}
By \cref{eq:fixpoint}, $W_{r+1}$ is a subset of $W_r$ by construction; under
keep-best, a table that converged in round $r$ is excluded from all later $W_{r'}$,
$r'>r$, so no element re-enters. A non-increasing sequence of finite sets bounded
below by $\emptyset$, evaluated for at most $R$ rounds, terminates. The depth label
reads off $W_R$.
\end{proof}

\paragraph{Join validation.} For a candidate join $\tau_a.x\to\tau_b.y$ proposed
by any of four signals (same name, shared skill-entity binding, data-driven
identifier discovery, or an LLM proposal told that names may differ), accept iff
\begin{equation}
\mathrm{cover}(x\!\to\!y)=\frac{|\mathrm{distinct}(x)\cap\mathrm{distinct}(y)|}
{|\mathrm{distinct}(x)|}\ge 0.5,
\label{eq:overlap}
\end{equation}
the column types are compatible, and $\tau_b.y$ is either data-measured
\textsf{unique} or a single-column primary key (a member of a \emph{composite} key is
rejected, being non-unique alone). Coverage is directional, the fraction of
child-key values with a referent in the candidate parent, so the $0.5$ floor
tolerates the partial referential integrity common in raw enterprise extracts, while
the uniqueness requirement on $\tau_b.y$ fixes the join's direction. Value overlap
establishes the \emph{precision} of accepted joins; recall is bounded by the four
candidate generators, and we do not claim a complete relationship graph. Each
accepted relationship carries its coverage, inferred cardinality, a confidence, and
an \evsql{} witness. The output
$\mathcal{K}(c,d)=(\text{per-table knowledge},\text{validated joins},
\text{glossary})$ is persisted durably; a compact brief $\mathrm{brief}(\mathcal{K})$
is shared by the report planner and the question generator.

\subsection{Admissibility: the trust contract}
A report view $v$ defines metrics $\{(\ell,\mathrm{ref}=(\textit{step},
\textit{col},\textit{agg}),\textsc{evidence\_sql})\}$. Let
$\mathit{derived}=\textit{agg}(\textit{step}.\textit{col})$ be the value the view
computed and $\mathit{ev}=\mathrm{scalar}(\mathrm{run\_sql}(\textsc{evidence\_sql}))$
the value obtained by independently re-executing the evidence query.
\begin{definition}[Admissible view]\label{def:adm}
$v$ is \emph{admissible} iff for every metric
\begin{equation}
\frac{|\mathit{derived}-\mathit{ev}|}{\max(|\mathit{derived}|,|\mathit{ev}|,\epsilon)}
\le 0.01,
\label{eq:adm}
\end{equation}
and every chart/table reference resolves to a defined step and a real column.
\end{definition}
A tolerance is needed because $\mathit{derived}$ and $\mathit{ev}$ follow
\emph{different} execution paths (recipe-step aggregation versus a standalone scalar
query) and can differ by float accumulation and rounding; the $1\%$ relative bound
absorbs this while still catching the order-of-magnitude and wrong-column errors that
matter.\footnote{The $1\%$ bound is a configurable default and a small positive
$\epsilon$ floors the denominator against division by zero; the guarantee boundary
is thus a parameter, not a universal constant.} A non-admissible view is
\emph{suppressed}; a report whose views are all
non-admissible is \emph{deferred} (hidden). The published artifact is therefore a
\emph{durable recipe} $r(v)$ (a SQL-step DAG, metric definitions, deterministic
chart specifications, and filters), and refresh is the \emph{zero-LLM, zero-kernel}
evaluation $\mathrm{exec}(r,\mathit{params})$; re-narration fires only when
$\mathrm{fingerprint}(\mathrm{exec}(r,\varnothing))$ changes.

\paragraph{The proactive loop.} Onboarding induces the chain
\[
\begin{aligned}
&\textsf{upload}\to \textsf{fast}(\Sigma)\to
\mathcal{K}=\mathrm{compile}(\Sigma,\sigma)\to \textsf{persist}\,\sigma\\
&\quad\to \textsf{Reports}=\{v\ \text{admissible}: v\in Y(\sigma)\ \text{over}\ \mathcal{K}\}\\
&\quad\to \textsf{Questions}=\mathrm{gen}(\sigma,\mathrm{brief}(\mathcal{K}),\textsf{Reports}),
\end{aligned}
\]
and a user click on a published number $n$ in view $v$ maps to a natural-language
query $q(n,v)$ that re-enters $P$, grounded in the same $\mathcal{K}$ (whose
validated joins are authoritative in \textsf{collate}). The system is proactive in
that it emits $(\textsf{Reports},\textsf{Questions})$ \emph{before} any $q$, and
reactive deep dives are $q\mapsto P(q\given \mathcal{K},\sigma)$.

\section{The Skill Abstraction}
\label{sec:skill}

A skill is the unit that makes domain expertise portable. It is a folder (not a
prompt, not a fine-tune, not a database row), so it is version-controlled,
reviewable, and shippable like code, yet it influences every reasoning step in the
system. \Cref{fig:skill} shows its anatomy.

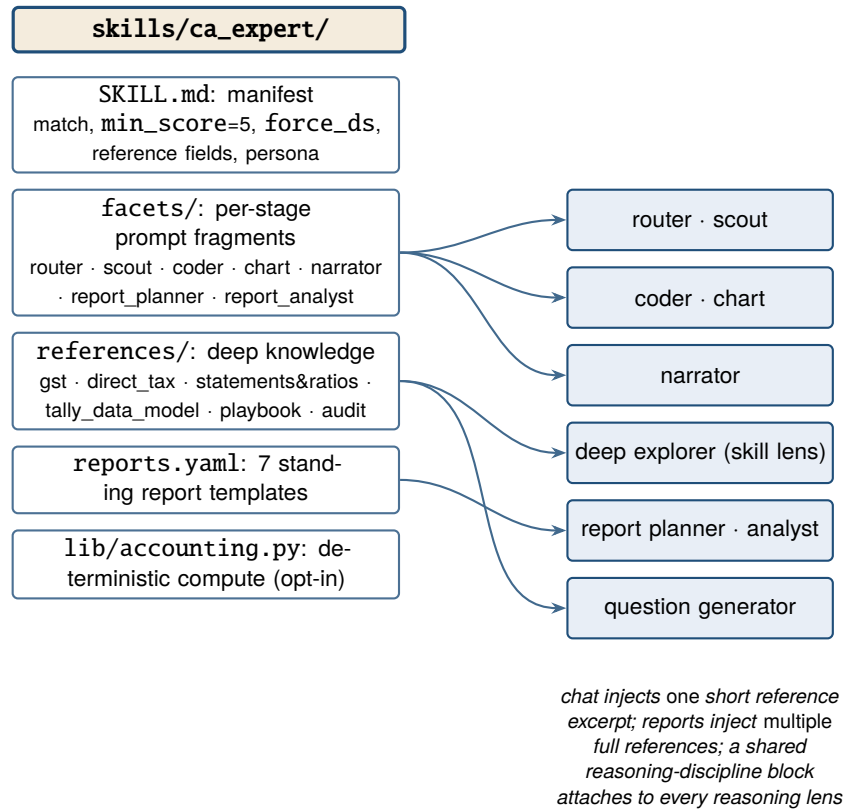
\begin{figure}[t]
\centering
\begin{tikzpicture}[node distance=3mm]
\node[skillb, anchor=north west, text width=49mm] (root) at (0,0)
  {\textbf{\code{skills/ca\_expert/}}};
\node[proc, below=of root.south west, anchor=north west, text width=49mm] (md)
  {\code{SKILL.md}: manifest\\\scriptsize match, \code{min\_score}=5, \code{force\_ds}, reference fields, persona};
\node[proc, below=2mm of md, text width=49mm] (facets)
  {\code{facets/}: per-stage prompt fragments\\\scriptsize router · scout · coder · chart · narrator · report\_planner · report\_analyst};
\node[proc, below=2mm of facets, text width=49mm] (refs)
  {\code{references/}: deep knowledge\\\scriptsize gst · direct\_tax · statements\&ratios · tally\_data\_model · playbook · audit};
\node[proc, below=2mm of refs, text width=49mm] (yaml)
  {\code{reports.yaml}: 7 standing report templates};
\node[proc, below=2mm of yaml, text width=49mm] (lib)
  {\code{lib/accounting.py}: deterministic compute (opt-in)};
\node[stage, right=22mm of facets.north east, anchor=north west, text width=33mm] (cons1) {router · scout};
\node[stage, below=2mm of cons1, text width=33mm] (cons2) {coder · chart};
\node[stage, below=2mm of cons2, text width=33mm] (cons3) {narrator};
\node[stage, below=2mm of cons3, text width=33mm] (cons4) {deep explorer (skill lens)};
\node[stage, below=2mm of cons4, text width=33mm] (cons5) {report planner · analyst};
\node[stage, below=2mm of cons5, text width=33mm] (cons6) {question generator};
\draw[flow] (facets.east) to[out=0,in=180] (cons1.west);
\draw[flow] (facets.east) to[out=0,in=180] (cons2.west);
\draw[flow] (facets.east) to[out=0,in=180] (cons3.west);
\draw[flow] (refs.east) to[out=0,in=180] (cons4.west);
\draw[flow] (yaml.east) to[out=0,in=180] (cons5.west);
\draw[flow] (refs.east) to[out=0,in=180] (cons6.west);
\node[lbl, below=6mm of cons6] (note) {\parbox{38mm}{\centering chat injects \emph{one} short reference excerpt; reports inject \emph{multiple} full references; a shared reasoning-discipline block attaches to every reasoning lens}};
\end{tikzpicture}
\caption{Anatomy of a domain-expert skill pack (the chartered-accountant skill) and
where each part is injected. The same folder feeds the live pipeline, the offline
explorer, both report engines, and the question generator.}
\label{fig:skill}
\end{figure}

\subsection{Anatomy of a skill pack}
A pack contains a \code{SKILL.md} manifest (YAML frontmatter plus a body), a
\code{facets/} directory of per-stage prompt fragments, a \code{references/}
directory of deep-knowledge documents, a \code{reports.yaml} of standing-report
templates, and an optional \code{lib/} of deterministic Python. The manifest
declares only \emph{lightweight routing configuration}: the auto-detect patterns,
a keyword-to-reference map, which references are ``core'' (always loaded for
reports), which reference supplies the KPI vocabulary and which the review cadence,
a persona example for narration, and a flag that routes the coder to the
data-science agent so the skill's modeling methods are available. The heavy text
(facets, references) is loaded into an in-memory object; the run context carries
only the slug. We ship three packs (a Chartered Accountant expert, a Chief
Marketing Officer expert, and a Supply-Chain expert) to demonstrate that the
framework is not hardcoded to one domain; the structure, not the result, is what
the second and third packs evidence.

\begin{lstlisting}[language=,caption={The frontmatter of the chartered-accountant skill (excerpt). Routing fields and \texttt{match} patterns are declarative; the value tokens enable detection on opaque schemas.},label={lst:skillmd}]
name: ca_expert
force_ds_agent: true
persona_example: business owner or CFO
kpi_reference: financial_statements_and_ratios
playbook_reference: ca_practice_playbook
schema_mapping_reference: tally_data_model         # deep-explorer recognition lens
reference_keywords:
  indian_gst: [gst, gstr, itc, cgst, sgst, igst, hsn, "place of supply", "3b"]
  indian_direct_tax: [tds, tcs, "income tax", "advance tax", "44ab", "194", ...]
  financial_statements_and_ratios: [ratio, "balance sheet", roce, "schedule iii", ...]
match:
  min_score: 5
  any_table: [ledger, voucher, daybook, trial_balance, gstr, "_tds_", tally, ...]
  any_column: [ledger, voucher_type, voucher_no, debit, credit, ...]
  any_value: ["sundry debtors", "sundry creditors", "duties & taxes", ...]  # opaque-schema fallback
\end{lstlisting}

\subsection{Registry and selection}
The registry is a lazy, thread-locked process singleton built once by scanning the
\code{skills/} directory; it has no database dependency, so it is safe to import
from any node, including the headless report jobs, and multi-worker deployments
each build an identical registry from the same files. A malformed pack is logged
and skipped rather than breaking the rest.

Selection follows \cref{eq:select}. An explicit assignment (an admin override, or a
previously cached auto-detection) wins first; otherwise the system loads table and
column names, preferring a local parquet schema cache and falling back to a unified
cross-connector schema config, so detection is source-agnostic across uploaded
files, SAP, Snowflake, Postgres, and others, and runs the scorer of
\cref{eq:score,eq:qualify}. On a confident hit the result is persisted back as an
\code{auto\_detected} assignment so subsequent lookups are a single fast read and an
operator can see and override the choice. \Cref{alg:select} states the resolution.

\begin{algorithm}[t]
\DontPrintSemicolon
\caption{Deterministic, fail-open skill selection $\sigma(c,d)$}
\label{alg:select}
\KwIn{client $c$, dataset $d$; registry $K$}
\KwOut{active slug or $\bot$}
\If{\textnormal{not} \code{should\_use\_skills}$(c)$}{\Return $\bot$}
\If{cache hit for $(c,d)$}{\Return cached value}
\If{explicit assignment $A$ exists}{
  \lIf{$A$ disabled}{\Return $\bot$}
  \lIf{$A.\mathit{slug}$ valid}{\Return $A.\mathit{slug}$}
}
$(T,\mathit{Col})\leftarrow$ schema names (parquet cache, else connector config)\;
$m\leftarrow \code{match\_schema}(T,\mathit{Col},\mathit{Val})$ \tcp*{\cref{eq:score,eq:qualify,eq:select}}
\If{$m\neq\bot$}{persist $m$ as \code{auto\_detected}; \Return $m.\mathit{slug}$}
\Return $\bot$ \tcp*{any exception in the above also returns $\bot$}
\end{algorithm}

\subsection{Cross-cutting injection}
\Cref{eq:inject} is realized by one small \code{render\_*} helper per node, each a
no-op when no skill is active (\cref{prop:noop}). The router receives a domain lens
for question rewriting and table routing; each per-table scout prompt is annotated
with the columns the expert cares about; the coder system prompt is prefixed, ahead
of any client custom prompt, with a methodology block that bundles the coder and
chart facets, plus a single keyword-selected reference excerpt; and the narrator is
framed to brief the skill's persona. A subtlety that matters in practice: the
\emph{chat} path splices exactly one reference excerpt under a tight budget (latency
is bounded), whereas the offline report path uses the \emph{same} keyword routing
but loads multiple \emph{full} references under a far larger budget, because reports
are correctness-first batch jobs that can afford the context. Finally, a shared
\emph{reasoning-discipline} block, demanding explicit confidence levels, named
blind spots, confounders-before-conclusions, and a falsification test for each
headline claim, is inherited by every skill at every reasoning lens, attaching an
explicit confidence/blind-spot/falsification discipline to each piece of generated
reasoning. An optional,
off-by-default hook can additionally inject a skill's deterministic Python into the
kernel bootstrap, wrapped so that a broken library is non-fatal and the live
product is unchanged until an operator opts in.

\subsection{Toward a marketplace of expert skills}
\label{sec:marketplace}
The abstraction is deliberately open-ended, and three properties make the catalogue
\emph{extensible by addition} rather than by engineering. (i)~A skill is a
self-contained folder with no database dependency, so authoring a new expert is a
content task (writing the manifest, facets, references, and report templates), not
a code change; the registry discovers it on a directory scan. (ii)~Selection is a
deterministic $N$-way score with a stable tiebreak (\cref{eq:select}), so adding the
tenth or hundredth skill does not perturb the resolution of the first; the catalogue
scales without a combinatorial routing problem. (iii)~The no-op property
(\cref{prop:noop}) guarantees that an unmatched or malformed pack cannot degrade the
base product, so packs can be published, versioned, and retired independently and
safely. Together these turn the framework into a substrate for a \emph{marketplace of
domain experts}: each pack encodes the playbook of a senior practitioner (a
PhD-level or veteran industry specialist) and the three we ship (a Chartered
Accountant, a Chief Marketing Officer, and a Supply-Chain expert) are seeds of a
catalogue that could span equity research, actuarial and clinical analysis, FP\&A,
pricing and revenue management, risk and audit, and beyond. A tenant connects data,
the relevant expert is selected automatically, and the same machinery that serves
one expert then serves an arbitrary library of them. We make no claim to have populated
such a marketplace; the present catalogue is small and human-authored
(\cref{sec:limits}). We claim only that the architecture is built to host one.

\section{Persistent Agentic Schema Exploration}
\label{sec:explore}

Before the system can apply domain expertise, it must build structural knowledge of
the data, profiling each column, testing encoding hypotheses, and validating joins.
A fast one-shot pass writes a one-line description per table at onboarding; the
knowledge-compilation phase is a second, slower pass that actually digs in. Its design is borrowed directly from how a coding
agent explores a codebase: give a model a small set of tools, let it choose actions,
feed back observations, and let an \emph{independent} check decide when it is done
\citep{yao2023react,shinn2023reflexion}.

\begin{figure}[t]
\centering
\begin{tikzpicture}[node distance=6mm]
\node[proc, text width=30mm] (probe) {probe agent\\\scriptsize JSON action over DuckDB};
\node[proc, right=14mm of probe, text width=24mm] (submit) {\code{submit}\\\scriptsize terminal action};
\node[dec, right=12mm of submit, text width=15mm] (cov) {coverage\\critic\\(code)};
\node[proc, right=12mm of cov, text width=22mm] (assemble) {assemble +\\\scriptsize upsert to store};
\draw[flow] (probe) -- node[lbl,above]{tool result} (submit);
\draw[flow] (submit) -- (cov);
\draw[flow] (cov) -- node[lbl,above]{ok} (assemble);
\draw[flow] (cov.south) to[out=-90,in=-90] node[lbl,below]{gaps: ``you missed column X''} (probe.south);
\node[lbl, above=8mm of submit] (guards) {\parbox{52mm}{\centering \textbf{anti-runaway guards}: doom-loop ($3\times$ same action) · diminishing-returns nudge · step ceiling · history trim}};
\draw[dflow] (guards.south) -- (submit.north);
\node[proc, below=14mm of probe.south west, anchor=north west, text width=22mm] (w1) {$W_1$: $|T|$ tables};
\node[proc, right=6mm of w1, text width=20mm] (w2) {$W_2\subseteq W_1$};
\node[proc, right=6mm of w2, text width=20mm] (w3) {$W_3\subseteq W_2$};
\node[stage, right=8mm of w3, text width=30mm] (depth) {\msc{deep} iff $W_R=\emptyset$\\else \msc{partial}};
\draw[flow] (w1) -- node[lbl,above]{round} (w2);
\draw[flow] (w2) -- node[lbl,above]{round} (w3);
\draw[flow] (w3) -- (depth);
\node[lbl, below=1mm of w2] {self-healing convergence (\cref{prop:term}), $R\le 3$};
\node[proc, below=10mm of w1.south west, anchor=north west, text width=43mm] (sig) {4 candidate signals\\\scriptsize same-name · skill-entity · data-driven id · LLM (``names may differ'')};
\node[dec, right=10mm of sig, text width=24mm] (gate) {overlap $\ge 0.5$\\type-compat.\\parent unique};
\node[stage, right=10mm of gate, text width=33mm] (rel) {Relationship\\\scriptsize \{overlap, cardinality, confidence, evidence\_sql\}};
\draw[flow] (sig) -- node[lbl,above]{DuckDB value-overlap} (gate);
\draw[flow] (gate) -- node[lbl,above]{accept} (rel);
\end{tikzpicture}
\caption{Knowledge compilation. Top: the per-table probe loop, where ``continue'' is the
default and ``stop'' is an explicit terminal action gated by a \emph{deterministic}
coverage critic, bounded by anti-runaway guards. Middle: the self-healing
convergence loop, whose pending set provably shrinks. Bottom: join discovery, where
every candidate is \emph{proven} on the data by value overlap.}
\label{fig:compile}
\end{figure}
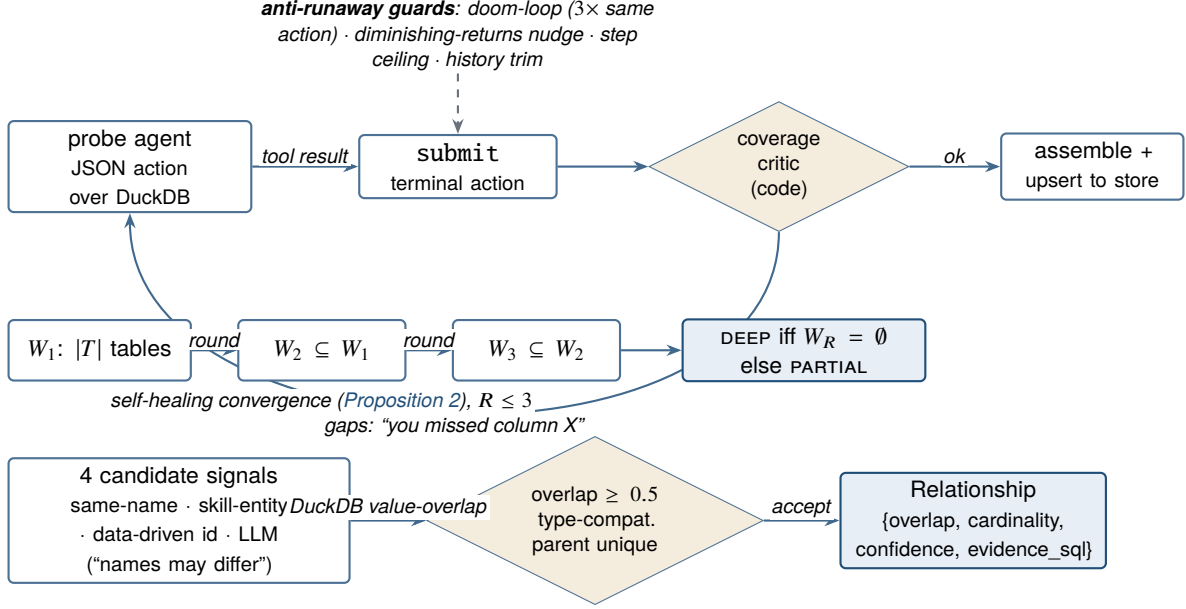

\subsection{The per-table probe loop}
One agent runs per table (fanned out with bounded concurrency). It is handed the
table's columns and types, the tool contract, and, if a skill matched, an expert
``lens'' describing what matters in this domain. It replies with a single action in
JSON (the providers do not reliably support native tool calling, so actions are
dispatched from JSON, the same mechanism the live coder uses); the system executes
that action against DuckDB over the parquet files (\emph{zero} load on the client's
production database) and returns the result as the next observation. The agent
profiles columns, samples values, tests encoding hypotheses (``is this period field
really \code{YYYYMM} text?''), and checks keys, until it calls the terminal
\code{submit} action with a full characterization.

\paragraph{Stop is gated, not self-declared.} Crucially, the decision to stop is
not the model's to make alone. A \emph{deterministic} coverage critic (plain code,
not an LLM, so it is fast, free, and never hallucinates) checks whether every
column was profiled and typed, whether a primary key and grain were identified, and
whether the skill's expected entities were found. If something is missing it returns
a specific gap (``you missed column $X$'') and the loop continues; only when the
coverage critic is satisfied is the table assembled and upserted. An \emph{optional}
LLM ``depth'' critic may then require a second, deeper pass; it is fail-open
(treating any error as satisfied), so it can raise quality but never blocks
termination (\cref{prop:term}). This mirrors the ``verification gate before done''
pattern and directly addresses the finding that intrinsic self-assessment is
unreliable \citep{huang2023cannotselfcorrect}: the in-loop gate is external and
deterministic. Anti-runaway guards (breaking on the same action
repeated three times, a diminishing-returns nudge, a per-table step ceiling, and
trimming of old observations so context stays small) bound the loop, and a
keep-best policy ensures the self-healing re-attempts of \cref{eq:fixpoint} never
regress a converged table. The honest depth label follows \cref{prop:term}: a run is
\msc{deep} only if \emph{every} table converged, else \msc{partial}.

\subsection{Joins proven on the data}
Cross-table joins are where naive schema reasoning fails, because matching columns
frequently have \emph{different names} across tables. The explorer generates
candidates from four generic signals (identical names, a shared skill-entity
binding, data-driven discovery of identifier-like columns, and an LLM proposal
explicitly told that names may differ and to match by meaning and overlapping
sample values) and then \emph{proves} each one on the data with a DuckDB
value-overlap query (\cref{eq:overlap}), keeping only joins whose values actually
overlap, whose types are compatible, and whose parent side has \emph{measured}
uniqueness. The use of measured cardinality rather than a declared primary-key flag
is deliberate: a member of a composite key is not unique on its own and must not be
treated as a parent. Each surviving relationship carries its overlap percentage,
cardinality, a confidence, and an evidence query, and these validated joins are
later treated as authoritative by the live pipeline's \textsf{collate} stage.

\subsection{Persistence and economics}
Findings are written to a per-(client, dataset) document in the operational store
the moment each table is done (durable and resumable) and exported to a JSON file
for backup. Human curation (notes on a database, table, or column, and manually
edited relationships) is preserved across re-runs, and refreshes are capped (a small
fixed number) to bound cost; the live pipeline reads the knowledge through a short
in-process cache. The detected skill is persisted back to the assignment store so
the live chat selector (which keys off names) and the value-based detector used here
agree, an edge in the onboarding chain we revisit in \cref{sec:proactive}. The cost
of compilation scales with the number of tables and their width, \emph{not} with row
count, because the data probing is DuckDB work and only the agent's reasoning
consumes tokens; it is a one-time, amortized cost rather than a per-query one.

\section{Execution-Verified Expert Report Synthesis}
\label{sec:reports}

Standing expert reports are the proactive product's centerpiece: pre-computed,
domain-expert views a user reads \emph{first}. Because they are shown without a
prompting question, correctness matters more than latency, and the engine is
modeled on the explorer (a planned, critic-gated agent) rather than on the live
coder. \Cref{fig:report} shows the pipeline; \cref{lst:report} shows a template
view; the trust core is \cref{def:adm}.

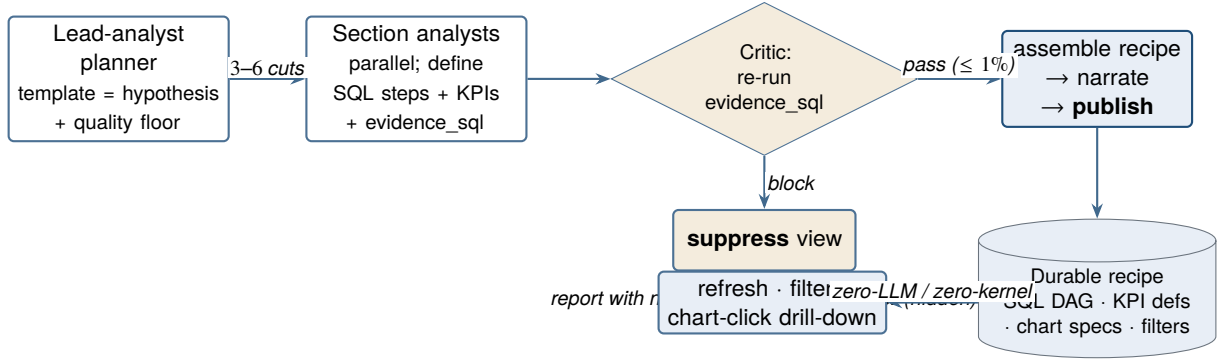
\begin{figure}[t]
\centering
\begin{tikzpicture}[node distance=8mm]
\node[proc, text width=27mm] (plan) {Lead-analyst\\planner\\\scriptsize template = hypothesis + quality floor};
\node[proc, right=10mm of plan, text width=27mm] (sec) {Section analysts\\\scriptsize parallel; define SQL steps + KPIs + evidence\_sql};
\node[dec, right=11mm of sec, text width=20mm] (crit) {Critic:\\re-run\\evidence\_sql};
\node[stage, right=11mm of crit, text width=23mm] (pub) {assemble recipe\\$\to$ narrate\\$\to$ \textbf{publish}};
\node[proc, below=7mm of crit, text width=22mm, fill=warm] (sup) {\textbf{suppress} view};
\draw[flow] (plan) -- node[lbl,above]{$3$–$6$ cuts} (sec);
\draw[flow] (sec) -- (crit);
\draw[flow] (crit) -- node[lbl,above]{pass ($\le 1\%$)} (pub);
\draw[flow] (crit) -- node[lbl,right]{block} (sup);
\node[lbl, below=2mm of sup] {report with no admissible view $\Rightarrow$ \textbf{defer} (hidden)};
\node[store, below=12mm of pub, text width=30mm, fill=soft] (rec) {Durable recipe\\\scriptsize SQL DAG · KPI defs · chart specs · filters};
\draw[flow] (pub) -- (rec);
\node[proc, left=12mm of rec, text width=28mm, fill=soft] (drill) {refresh · filter ·\\chart-click drill-down};
\draw[flow] (rec) -- node[lbl,above]{zero-LLM / zero-kernel} (drill);
\end{tikzpicture}
\caption{Execution-verified report synthesis. Every metric carries an evidence
query that an independent, non-LLM critic re-executes; a view that fails
\cref{def:adm} is suppressed, and a report with no admissible view is deferred. The
published artifact is a durable recipe that refreshes and drills down at zero LLM
and zero kernel cost.}
\label{fig:report}
\end{figure}

\subsection{Plan, fan out, verify, narrate}
An adaptive lead-analyst planner treats a \code{reports.yaml} template as a
\emph{hypothesis and a quality floor}, not a script: it probes the data and expands
the report into a handful of verifiable cuts that the actual columns support. A
section analyst is then dispatched per view (bounded concurrency), each defining
durable SQL steps and emitting metrics with an evidence query attached. The
authoritative step is the critic: for each metric it \emph{independently}
re-executes the evidence SQL and tolerance-matches the result against the value the
view derived (\cref{eq:adm}); a value/evidence mismatch, a broken reference, or an
evidence-query error is a \emph{block} (a non-numeric result is a block for a metric
that declares a numeric unit). This is execution-grounded verification in the spirit
of CRITIC and self-debugging \citep{gou2023critic,chen2023selfdebug}, applied not to
fix code but to \emph{gate publication}: a number that does not reproduce against its
own evidence query is never shown.

\paragraph{The failure mode is silence, not error.} A view that fails verification
is suppressed; a report with no surviving view is deferred and simply not displayed.
The system never surfaces a raw error to a user, and, within the boundary stated
below, never publishes a number it could not reproduce. The honest claim is not
``zero hallucination'' or ``always correct''; it is \emph{suppress-or-defer}: an
unverifiable figure is withheld. And the boundary is precise: \cref{eq:adm}
certifies that a published number \emph{matches its declared evidence SQL}, not that
the metric's business definition is the semantically right one. Verification guards
arithmetic faithfulness, not definitional correctness.

\begin{lstlisting}[language=,caption={A standing-report template view (the GST summary, excerpt). The template is client-agnostic instructions; the engine maps it onto whatever books the client has, then verifies and saves a re-runnable recipe.},label={lst:report}]
- id: gst_summary_reconciliation
  category: GST
  applies_when: Dataset has GST fields (CGST/SGST/IGST, taxable value) ...
  kpi_hints: [output_tax, input_tax, net_gst_payable, itc_available]
  views:
    - id: overview
      prompt: >
        As a GST expert, summarise the GST position for the tax period. Map the
        schema to taxable value, CGST/SGST/IGST and sales-vs-purchase direction.
        Report output tax, input tax/ITC and NET GST payable, split by rate. Where
        a portal figure (GSTR-2B) is absent, DECLARE the blind spot rather than
        assume all input tax is creditable. ...
      sections: [kpi_strip, chart, table, insight]
\end{lstlisting}

\subsection{The durable recipe}
Publication produces a \emph{recipe}: a SQL-step DAG, metric definitions (each a
reference into a step plus an aggregation and its evidence query), deterministic
chart specifications, a drill-down table, and filters. The recipe, not an LLM, is
what refresh re-runs: refreshing, filtering, and chart-click drill-down all evaluate
the recipe with zero LLM and zero kernel cost, and because charts, KPI strip, and
table all read from the \emph{same} step rows, they cannot disagree. Re-narration is
gated by a data fingerprint, so the language model is invoked again only when the
underlying numbers actually change. The published document is shaped to match the
platform's existing report contract exactly, carrying the recipe, a ``how we
computed this'' provenance view that exposes the evidence SQL, and a per-view
confidence tier additively, so the engine is a drop-in behind a flag, fully
backward-compatible with the legacy generator.

\section{Proactive Question Generation and the Loop}
\label{sec:proactive}

The final step is to tell the user \emph{what to ask}. Suggested questions are
generated from two fused sources: the matched skill's report themes, KPI catalogue,
and review-cadence playbook (the expert's agenda), and the \emph{same}
validated-knowledge brief the report planner consumes (so the questions are
answerable against real entities, grains, and joins). The generator builds roughly
thirty diverse questions, then rewrites them into standalone analyst phrasing using
business vocabulary (never raw column names) and writes them to a cache. The prompt
targets the review-cadence questions a senior domain leader would ask to run the
business, scoped to the entities the schema exposes; we make no claim about how end
users perceive them. Everything is
fail-open: missing metadata, skill, or knowledge degrades to schema-only questions,
and finally to a fixed set of generic fallbacks, so the home page is never empty.

\begin{figure}[t]
\centering
\begin{tikzpicture}[node distance=10mm]
\node[stage, text width=24mm] (data) {Data connected /\\uploaded};
\node[stage, right=14mm of data, text width=30mm] (know) {Fast + Deep exploration\\\scriptsize validated knowledge $\mathcal{K}$, skill auto-persisted};
\node[stage, right=14mm of know, text width=26mm] (rep) {Standing expert reports\\\scriptsize verified KPIs};
\node[stage, below=10mm of rep, text width=26mm] (q) {Suggested questions\\\scriptsize mirror report KPIs};
\node[stage, left=14mm of q, text width=26mm] (read) {User reads numbers,\\clicks a figure};
\node[proc, left=14mm of read, text width=24mm, fill=warm] (live) {Live pipeline $P$\\\scriptsize verified deep dive};
\draw[flow] (data) -- node[lbl,above]{} (know);
\draw[flow] (know) -- node[lbl,above]{\textsc{e1} fast$\to$deep} (rep);
\draw[flow] (know.south) to[out=-90,in=90] node[lbl,right]{\textsc{e2} persist skill} ($(know.south)+(0,-7mm)$);
\draw[flow] (rep) -- node[lbl,right]{\textsc{e4} on success} (q);
\draw[flow] (q) -- (read);
\draw[flow] (read) -- node[lbl,above]{composed NL prompt} (live);
\draw[flow] (live.north) to[out=90,in=180] node[lbl,above]{$q\given \mathcal{K},\sigma$} (rep.west);
\node[lbl, below=2mm of know] {\textsc{e3} deep$\to$reports (only on convergence)};
\node[lbl, below=12mm of read] {\parbox{72mm}{\centering No prior user query is required: the loop runs to completion before the blank query box is ever used.}};
\end{tikzpicture}
\caption{The proactive-analyst loop and the fully-automatic onboarding chain.
Four trigger edges (\textsc{e1}--\textsc{e4}) wire onboarding end-to-end; a clicked
number is handed back to the live pipeline as an on-topic deep-dive query grounded
in the same compiled knowledge.}
\label{fig:proactive-loop}
\end{figure}

\subsection{The fully-automatic onboarding chain}
What makes this a \emph{loop} rather than a set of features is that the edges fire
automatically (\cref{fig:proactive-loop}). Connecting or uploading data triggers the
fast pass; \textsc{e1} the fast pass triggers deep exploration; \textsc{e2} deep
exploration persists the detected skill into the assignment store (bridging the
name-based chat selector and the value-based deep detector so both agree); \textsc{e3}
deep exploration triggers report generation \emph{only after} convergence; and
\textsc{e4} successful report generation triggers question regeneration so the
questions mirror the freshly verified report KPIs. Each edge is fire-and-forget and
fail-open, and autonomy is human-overridable throughout: an administrator can assign
or disable a skill, edit knowledge notes, and trigger a capped refresh.

\subsection{Closing the loop in the interface}
The user lands not on a blank box but on a board of finished analyses: an operational
brief, a hero summary, alert cards, each with a small visualization rendered from the
\emph{real} published figure. Reading a number raises the natural follow-up
``why?'', and there are two paths. The first is an in-page, zero-LLM drill: clicking
a chart re-evaluates the recipe at a finer grain. The second hands control back to
the live pipeline: a composed natural-language prompt is marshaled into the streaming
chat, where it is answered against the same validated knowledge, closing the loop
from \emph{reports} to \emph{numbers} to \emph{questions} to \emph{verified deep
dive}. Because the suggestions come from a domain-expert prior that exists on day one
and are generated against the auto-explored schema, they sidestep both failure modes
of prior art: they need no query logs \citep{chatzopoulou2009querie,wang2022nextstep}
and they are role-relevant rather than merely statistically salient
\citep{ding2019quickinsights}.

\section{System and Implementation}
\label{sec:system}

We describe the system at the level of its \emph{principles and contracts} (the
abstractions, the selection and admissibility rules, and the engineering posture
that make the architecture work) rather than as a turnkey blueprint; internal
prompt text, tuning constants, and module-level wiring are elided as
implementation detail. With that scope set, the live regime is a custom, durable,
multi-stage streaming pipeline (the stages of \cref{sec:formal}). A run is a
detached, durable task that writes
its events to a per-run append-only log; the HTTP response body is a replay-then-tail
reader over that log, which makes reconnection and resume natural and decouples the
long-running analysis from event delivery. The active skill is resolved exactly once,
in \textsf{guard}, in parallel with relevance guards, and threaded downstream as a
lightweight slug. Multi-tenancy is strict: every datastore query, storage path, and
job is scoped by a client identifier with no defaults, table-level access control is
enforced at resolution and execution, and model selection is per-client and per-tier
behind a rate governor.

A single \emph{substrate}, parquet over DuckDB, is shared by the explorer, the
report engine, and the live coder, so exploration imposes zero load on the client's
production database and the live path benefits from the same fast columnar access; an
analytics backend abstraction prefers materialized parquet and falls back to
read-only generated SQL across the supported relational connectors. Charts are rendered
deterministically backend-side from a typed specification rather than from free-hand
plotting code, which is what lets a report's recipe reproduce identical visuals
without an LLM.

The engineering posture is uniform: \emph{fail-open at every boundary} and
\emph{additive behind flags}. The skill framework, deep exploration, the deep report
engine, and the kernel-library hook are each gated, default to the prior behavior,
and degrade to a no-op on error (\cref{prop:noop}); document shapes are
backward-compatible; and the system is covered by a large test suite plus a
DuckDB-oracle correctness harness. The cost/latency trade-off is made explicitly by
path: the chat path uses a single small reference excerpt and is latency-bounded,
while exploration and report generation are quality-first batch jobs that can afford
full references and can wait out rate limits.

\section{Case Study: Single-Tenant Deployment}
\label{sec:casestudy}

We report illustrative observations from one live tenant onboarded with Tally
accounting data materialized to parquet. We emphasize that these are
\emph{single-deployment} figures, presented for transparency; we ran no controlled
study and report no benchmark numbers.

\begin{table}[H]
\centering
\caption{Single-tenant deployment evidence (illustrative, observed from one live
account). These counts evidence end-to-end execution and convergence on one schema;
they are \emph{not} precision/recall of join discovery or report correctness, for
which no ground truth was collected.}
\label{tab:deployment}
\small
\begin{tabularx}{\linewidth}{@{}lX@{}}
\toprule
\textbf{Aspect} & \textbf{Observed} \\
\midrule
Skill auto-detected & \code{ca\_expert} (Chartered Accountant), via \emph{value-based}
match on an opaque/Tally schema where names alone would not fire \\
Knowledge depth & \msc{deep} (every probed table converged; \cref{prop:term}) \\
Tables converged & $10\,/\,10$ \\
Validated joins & $22$, each carrying an evidence query (\cref{eq:overlap}) \\
Standing reports generated & $7$ chartered-accountant reports, with \emph{no} prior
chat query \\
Example evidence-reproduced figure & a GST report \emph{published} a net figure of
$\approx\!\text{\rupee}97.86$\,L, reproduced against its evidence SQL within the
system (\cref{def:adm}), a published, reproduced value, not an independently audited
financial fact \\
Suggested questions generated & $\approx\!30$ accounting questions, grounded in the
report KPIs and the validated-knowledge brief \\
Onboarding & fully automatic (four fire-and-forget trigger edges) \\
Refresh marginal cost & zero-LLM, zero-kernel via the durable recipe \\
Regression posture & the test suite passes with the skill, explorer, and report
subsystems active \\
\bottomrule
\end{tabularx}
\end{table}

Qualitatively, the loop ran end to end with no human in the path: the value-based
detector recognized accounting data from distinctive ledger-group strings even though
the table names were opaque, exploration converged on all ten tables and proved
twenty-two joins, seven reports were authored and admitted under the verification
gate, and roughly thirty domain questions were generated (phrased in business
vocabulary and grounded in the report KPIs and validated-knowledge brief) before the
tenant ran a single query. The published GST figure illustrates
the trust contract: it appears \emph{because} its evidence query reproduced it within
tolerance, and views whose numbers could not be reproduced would have been suppressed
rather than shown.

\section{Discussion: Novelty and Positioning}
\label{sec:discussion}

The defensibility of this architecture is structural, and it survives
counter-positioning along three axes (\cref{tab:positioning}).

\begin{table}[t]
\centering
\caption{Positioning along three axes: who supplies the question, what
``proactive'' means, and what guarantees the number.}
\label{tab:positioning}
\footnotesize
\renewcommand{\arraystretch}{1.18}
\begin{tabularx}{\linewidth}{@{}l>{\raggedright\arraybackslash}X>{\raggedright\arraybackslash}X>{\raggedright\arraybackslash}X@{}}
\toprule
\textbf{System / class} & \textbf{Question source} & \textbf{Proactivity} &
\textbf{Number guarantee} \\
\midrule
Text-to-SQL; question-first NLQ \citep{yu2018spider,li2023bird,aws2021quicksightq,google2025lookerca}
& user supplies the question & reactive & translation correctness only \\
Log/next-step recommenders \citep{chatzopoulou2009querie,wang2022nextstep}
& mined from query logs & reactive-assist; cold-starts on new data & none \\
Statistical insight miners \citep{ding2019quickinsights,demiralp2017foresight,vartak2015seedb,thoughtspot2025spotiq,microsoft2025quickinsights}
& automatic & proactive but \emph{statistical} & statistical interestingness \\
Metric-bound proactive \citep{tableau2025aboutpulse,tableau2025pulseinsighttypes}
& automatic/triggered & proactive but bound to \emph{curated metrics} & statistical, LLM-narrated \\
Agent skills / skill libraries \citep{anthropic2025agentskills,wang2023voyager}
& general tooling & modular, not per-tenant analytics & varies \\
\textbf{This work} & \textbf{system supplies, from a portable expert prior} &
\textbf{proactive over un-curated, skill-covered schema} & \textbf{KPI re-verified by
evidence-SQL re-execution} \\
\bottomrule
\end{tabularx}
\end{table}

\paragraph{Who supplies the question.} Text-to-SQL and question-first NLQ assume the
user has a question; log-based recommenders supply one but need prior logs and an
active session, so they cold-start on a fresh tenant. We supply the question from a
\emph{portable domain-expert prior} that exists on day one with no logs and no
per-client engineering, and that is guaranteed answerable because it is generated
from the same auto-explored, validated knowledge the pipeline answers against.

\paragraph{What ``proactive'' means.} Statistical insight miners surface what is
statistically salient, not what a chartered accountant or marketing leader would act
on; the most proactive metric-driven product, Tableau Pulse, is bound to an
analyst-curated metric layer \citep{tableau2025aboutpulse}. We are proactive over an
un-curated schema (for a domain a skill covers) with \emph{no} pre-curated metric
layer: the expert supplies the agenda from the skill's templates and the explorer
expands it into the cuts the actual columns support.

\paragraph{What guarantees the number.} Persona/expert injection is known to
introduce domain-correlated bias \citep{tseng2024persona}, and intrinsic
self-correction is unreliable \citep{huang2023cannotselfcorrect}. Our countermeasure
is to fuse the opinionated expert with \emph{execution-grounded verification}: every
published metric carries an evidence query that an independent, deterministic critic
re-executes and tolerance-matches, and every join is proven on the data. The expert
sets the agenda; it cannot fabricate the figure.

\paragraph{The honest differentiator.} It is none of the commoditized parts in isolation: not
NL-to-SQL, not anomaly detection, not skill packaging. It is the \emph{combination},
deployed and fail-open: a portable domain-expert layer threaded as a cross-cutting
concern through router, scout, coder, chart, narrator, the explorer, and both report
engines; schema-derived (not metric-bound) agentic analysis; and end-to-end execution
verification, instantiated as a closed proactive loop with a fully-automatic
onboarding chain. We do not claim to invent skill packs (modular agent skills predate
and parallel this work \citep{anthropic2025agentskills}); the novelty is the
per-tenant, schema-auto-selected analytics \emph{specialization} fused with
execution-grounded trust.

\paragraph{The trade-off we accept.} A curated metric layer
\citep{dbt2024semanticlayer,looker2025lookml} encodes \emph{definitional}
correctness (blessed business definitions) that our schema-derived agenda forgoes.
We trade per-organization curation cost for portability and recover \emph{arithmetic}
(not definitional) trust via evidence re-execution; bridging the two (ingesting an
existing metric layer as an additional skill input) is natural future work.

\section{Limitations and Threats to Validity}
\label{sec:limits}

\begin{itemize}
\item \textbf{Single deployment, no controlled study.} Evidence is from one tenant; we
report no user study and no public-benchmark numbers. The deployment figures are
illustrative, and the two undeployed packs evidence the generality of the
\emph{structure}, not of results.
\item \textbf{Heuristic, human-authored skills.} The catalogue is small and the packs
are written by hand; auto-detection is a heuristic that can miss or mis-route on a
truly generic schema, mitigated, but not eliminated, by explicit assignment and
fail-open-to-no-skill (\cref{prop:noop}).
\item \textbf{Verification boundary.} \Cref{eq:adm} certifies that a published number
matches its declared evidence query within tolerance; it does \emph{not} certify that
the metric's business definition is semantically correct. The guarantee is arithmetic
faithfulness, not definitional correctness, and its failure mode is suppression, not a
certificate of universal correctness.
\item \textbf{Persona bias.} An opinionated expert can bias framing
\citep{tseng2024persona}; deterministic grounding (evidence re-execution, typed
formatting, deterministic rendering) constrains the \emph{numbers} but not necessarily
the emphasis.
\item \textbf{Cost and staleness.} Exploration and report generation are token-heavy
offline jobs whose cost scales with tables and width; it is amortized by zero-cost
refresh, not free. Compiled knowledge can go stale; it is regenerated on
schema-fingerprint change and via capped manual refresh, but is not continuously
revalidated.
\end{itemize}

\section{Conclusion}
\label{sec:conclusion}

A pluggable domain-expert skill abstraction and an offline knowledge-compilation loop
together flip enterprise data analysis from reactive question-answering to a
proactive, verified analyst: the system understands the data, authors the agenda,
re-verifies each published number against its evidence SQL, and tells the user what to
ask, before the query box is ever
used. The architecture's defensibility is not any one commoditized capability but
their fusion: domain expertise as a cross-cutting, auto-selected, fail-open concern,
joined to execution-grounded trust and closed into a proactive loop with an automatic
onboarding chain. Because expertise lives in self-contained, deterministically
selected, fail-open packs, the natural trajectory is a \emph{marketplace of domain
experts} (each the encoded playbook of a PhD-level or veteran industry
specialist) that a tenant draws on automatically the moment data is connected.
Future work is that larger and partly auto-authored catalogue, \emph{semantic} (not
only numeric) metric verification, continuous knowledge revalidation, and the
rigorous user study this single-deployment paper deliberately does not claim.

\small
\bibliographystyle{plainnat}
\bibliography{references}

\end{document}